\pdfoutput=1

\documentclass[11pt]{article}

\usepackage[]{acl}%review

\usepackage{times}
\usepackage{latexsym}

\usepackage[T1]{fontenc}

\usepackage[utf8]{inputenc}
\usepackage{microtype}
\usepackage{inconsolata}

\usepackage{amsmath}
\usepackage{graphicx}
\usepackage{multirow}
\usepackage{xcolor}
\usepackage{comment}
\usepackage{amsmath}
\usepackage{amssymb}
\usepackage{mathtools}
\usepackage{amsthm}

\theoremstyle{plain}
\newtheorem{theorem}{Theorem}[section]

\newtheorem{lemma}[theorem]{Lemma}
\newtheorem{corollary}[theorem]{Corollary}
\theoremstyle{definition}

\theoremstyle{remark}

\def\R{\mathbb{R}}

\def\dotprod#1#2{\langle#1,#2\rangle}

\def\vec#1{\mathbf{#1}}

\title{MeMo: Towards Language Models with Associative Memory Mechanisms
}
\author{
\textbf{
Fabio Massimo Zanzotto$^{1,3}$
Elena Sofia Ruzzetti$^1$
Giancarlo A. Xompero$^{3,1}$
}\\
\textbf{
Leonardo Ranaldi$^{2,1}$ 
Davide Venditti$^1$
Federico Ranaldi$^1$
}\\
\textbf{Cristina Giannone$^3$ 
Andrea Favalli$^3$ 
Raniero Romagnoli$^3$} \\
$^1$Human-centric ART, University of Rome Tor Vergata, Italy \\
$^2$University of Edinburgh, Scotland, UK \\
$^3$Almawave S.p.A., Via di Casal Boccone, 188-190 00137, Rome, IT
\\
\small{\href{mailto:fabio.massimo.zanzotto@uniroma2.it}{\color{black} \tt fabio.massimo.zanzotto@uniroma2.it}}
}

\begin{document}

\maketitle

\begin{abstract}
Memorization is a fundamental ability of transformer-based Large Language Models, achieved through learning. In this paper, we propose a paradigm shift by designing an architecture to memorize text directly, bearing in mind the principle that memorization precedes learning. We introduce MeMo, a novel architecture for language modeling that explicitly memorizes sequences of tokens in layered associative memories. By design, MeMo offers transparency and the possibility of model editing, including forgetting texts. We experimented with the MeMo architecture, showing the memorization power of the one-layer and the multi-layer configurations. 

\end{abstract}
\section{Introduction}
Transformer-based Large Language Models achieve unrivaled performance in language modeling by learning to capture and represent complex sequential dependencies from statistical patterns through extensive training phases that iteratively refine their weights to best approximate natural language. This has triggered significant interest in gaining a better understanding of the inner workings of these models, focusing on how these models generalize and capture structure between similar samples in terms of syntactic dependencies \cite{vig-belinkov-2019-analyzing}, compositional relations \cite{hupkes2020compositionalitydecomposedneuralnetworks,zanzotto-etal-2015-squibs} concerning the quantity \cite{Reizinger2024UnderstandingLR} and quality \cite{yang-etal-2024-unveiling} of training data. 

Besides generalization, a key component of Transformers' success is the ability to memorize data while learning \cite{ranaldi-etal-2023-precog,ranaldi-etal-2023-dark}. Indeed, earlier work investigated this other side of learning.
While \citet{carlini2023quantifyingmemorizationneurallanguage,mahdavi2024memorization} demonstrated evidence of memorization, \citet{kharitonov2021bpeaffectsmemorizationtransformers,mahdavi2024memorization} studied how the internal components lead to memorization, and \citet{kim2023provable} estimated the boundary between generalization and memorization, providing an estimation on their storage capacity. Memorization is not inherently a drawback in language models because it plays a crucial role in handling factual knowledge, which is important for question answering, summarization, or information retrieval. This factual recall relies on a delicate balance. While generalization helps capture patterns and unseen relationships in data,
memorization ensures that models retain critical and exact information when required.

Recent research has highlighted that memorization capability can be effectively harnessed using concepts rooted in associative memories \cite{kohonenCorrelationMatrixMemories1972a,andersonSimpleNeuralNetwork1972} - a system designed to link inputs to specific outputs and offers a structured and transparent way to store and retrieve information. 
By leveraging associative memory mechanisms, researchers have proposed strategies to post-edit LLMs \cite{meng2022ROME,meng2023memit}, enabling control over what is memorized, how it is stored, and how it is accessed, enhancing their reliability in fact-based tasks.

In this paper, we propose a paradigm shift by designing Language Models based on a different principle: memorization proceeds learning. By using associative memories, we build MeMo, a novel architecture for language modeling that explicitly memorizes sequences of tokens in layered associative memories. 
MeMo leverages correlation matrix memories \cite{kohonenCorrelationMatrixMemories1972a,andersonSimpleNeuralNetwork1972}, the concept that tokens and sequences of tokens can be represented as random vectors \cite{Plate1995,sahlgren05}, and the Johnson-Lindestrauss Transform to embed larger vectors in smaller spaces by preserving their distances \cite{JLL}. By design, MeMo offers transparency and the possibility of model editing, including forgetting texts. We experimented MeMo, showing the memorization power of single and multi-layer architecture.

\section{Preliminaries and Background}

\paragraph{Representing words or tokens in small random vectors} 
is the first important step in building language models with neural network architectures. Using random vectors is a standard technique. Indeed, random vectors are used in random indexing \cite{sahlgren05} in information retrieval to reduce the document vector space and in distributed representations for neural networks as a convenient way to determine a set of vectors to represent sets of different tokens \cite{Plate1995} or structures \cite{DBLP:conf/icml/ZanzottoD12,DBLP:conf/ijcnn/ZanzottoF17,zanzotto-etal-2020-kermit}. 
Moreover, random vectors are used to initialize weight matrices in any language-oriented application in neural networks, including the initialization of transformers \cite{10.5555/3295222.3295349} to build large language models from scratch.

Multivariate Gaussian random vectors have the important property of being able to generate sets $E$ of nearly orthogonal unitary vectors  that can form an approximate base of the space $R^n$ in a smaller space $R^d$ \cite{JLL}. Each token $t$ is then represented with a distinct vector in $\vec{t} \in E$, and the two following properties hold with a probability larger than $1 - \delta$:
$$
\begin{array}{cc}
\|\vec{a}^T\vec{b}\| < \epsilon   & \text{if } a \neq b \\
1 - \epsilon < \vec{a}^T\vec{b} < 1 + \epsilon       & \text{if } a = b 
\end{array}
$$
where $a$ and $b$ are tokens and $\vec{a}$ and $\vec{b}$ are vectors representing those tokens in the reduced space $R^d$.
By using the Johnson-Lindestrauss Lemma \cite{JLL}, it is possible to find a lower bound of how large $d$ should be in order to host $n$ vectors given the approximation $\epsilon$ and the probability factor $\delta$ (see Appendix A). In less precise equations, the two properties can be rewritten as:
$$
\vec{a}^T\vec{b} \approx \left\{\begin{array}{cc}
0   & \text{if } a \neq b \\
1   & \text{if } a = b 
\end{array}\right.
$$

Using these vectors with their properties, it is possible to represent a bag-of-tokens $B$ in a single vector $\vec{t_B}$ offering the operation that approximately counts the number of times a token is in $B$. The vector $\vec{t_B}$ is obtained by summing up vectors representing tokens in $B$ and, then, the counting operation is:
$$
\vec{a}^T\vec{t_B} \approx k  
$$
where $k$ is the number of times $a$ belongs to the bag $B$.

\paragraph{Correlation matrix memories (CMMs)} 
\cite{kohonenCorrelationMatrixMemories1972a,andersonSimpleNeuralNetwork1972} are a powerful tool to store key-value $(k_i,v_i)$ pairs in distributed memories as the sum of outer products of the vectors representing the keys $k_i$ and vectors representing the values $v_i$:  
\begin{equation}
C = \sum_{i=1}^n \vec{k}_i \vec{v}_i^T
\label{eq:rev}
\end{equation}
These CMMs have been generally defined on one-hot representations \cite{hobsonCorrelationMatrixMemories2011} and, eventually, reduced afterwards \cite{kohonenCorrelationMatrixMemories1972a}.  
Then, to retrieve the value associated with a key, the matrix $C$ should be multiplied with $\vec{k}_j^T$. As vectors $\vec{k}_i$ are one-hot vectors, the following property holds:
$$
\vec{k}_j^T C = \vec{v}_j
$$
To optimize the construction of these CMM matrices, we use the correlated form:
$$
C = K V^T = \left [ \begin{array}{cccc}
| & |& &| \\
\vec{k}_1&\vec{k}_2& \ldots &\vec{k}_n\\
| & |& &| \\
\end{array}
\right ]
\left [ \begin{array}{cccc}
  - & \vec{v}^T_1 & -  \\
  - & \vec{v}^T_2 & -  \\
 & \vdots &  \\

  - & \vec{v}^T_n & -  \\
\end{array}
\right ]
$$

To make CMMs practical, in MeMo, we use these memories along with the multivariate Gaussian vectors to represent keys and values. Hence, the generic property of this associative matrices is $$
\vec{k}_j^T C \approx \vec{e}_j^T V = \vec{v}_j
$$
where $\vec{e}_j$ is the onehot vector of the position $j$ and $\vec{k}_j$ and $\vec{v}_j$ are multivariate Gaussian vectors to represent the key $k_j$ and the value $v_j$.  

The idea behind correlation matrix memories has often been used to explain that feed-forward matrices are where Transformer architectures store most information \cite{mengMassEditingMemoryTransformer2023}. In MeMo, CMMs become the cornerstone for defining a novel approach to building Language Models.

\paragraph{Johnson-Lindestrauss Transform} \cite{JLLsimple_demonstration}, derived by using the Johnson-Lindestrauss Lemma (JLL) \cite{JLL}, guarantees that it exists a linear transformation $T_{d \times n}$ that transforms vectors in a bigger space $R^n$ in vectors in a smaller space $R^d$ by preserving their distance with an approximation $\epsilon$. Then,  given two vectors $\vec{a}$ and $\vec{b}$ in $R^n$,  the following property is guaranteed:
$$
\|\vec{a} -\vec{b}\| - \epsilon < \|T\vec{a} -T\vec{b}\|<\|\vec{a} -\vec{b}\| + \epsilon
$$
The JLL with the demonstration in \cite{JLLsimple_demonstration} shows that it is possible to build this matrix $T$ by using multivariate Gaussian vectors as transformation rows. 

JLT matrices are the last ingredient of our new model, as we need to transpose sequences of tokens in their representations in the target $R^d$ space.

\section{MeMo: Language Models with Multi-layer Correlation Matrix Memories}

Building on Correlation Matrix Memories, on multi-variate Gaussian vectors to represent tokens and token sequences, and on Johnson-Lindestrauss Transforms, we present here MeMo\footnote{MeMo is on \href{https://github.com/humancentricart/MeMo}{GitHub - HumanCentricART - MeMo}. MeMo is distributed under the license CC BY-NC-SA 4.0}, a way to build language models that memorize texts in a clear, transparent way. We first present how to build a language model with a single CMM (Sec. \ref{sec:single_LM-CMM}). This single-layer CMM language model predicts the next tokens of sequences with a fixed length $h$. Then, we generalize MeMo to a multi-layer approach in order to increase the length of the sequences that can be memorized, retrieved, and forgotten~(Sec.~\ref{sec:multi-layerCMM}).

\subsection{Language Models with single Correlation Matrix Memories}
\label{sec:single_LM-CMM}
Correlation matrix memories (CMMs) and multi-variate Gaussian vectors with their properties offer an interesting opportunity to build simple language models. 

Language models can be seen as predictors of the next tokens given input sequences. From a symbolic perspective, a language model stores the associations between sequences and the next tokens along with the observed frequency in order to estimate the probability. Then, from a symbolic perspective, the base for a language model is a multi-set $LM$ containing:
$$
LM = \{([x_1,x_2,...,x_h],y)\} = \{(s,y)\} 
$$
where $s = [x_1,x_2,...,x_h]$ are the fixed length sequences of tokens and $y$ are the next tokens implied by sequences $s$. Tokens are contained in a fixed vocabulary $V$ of $n$ tokens. These multisets are the sample sets where probabilities are estimated by counting.

\begin{figure}
  \includegraphics[width=\linewidth]{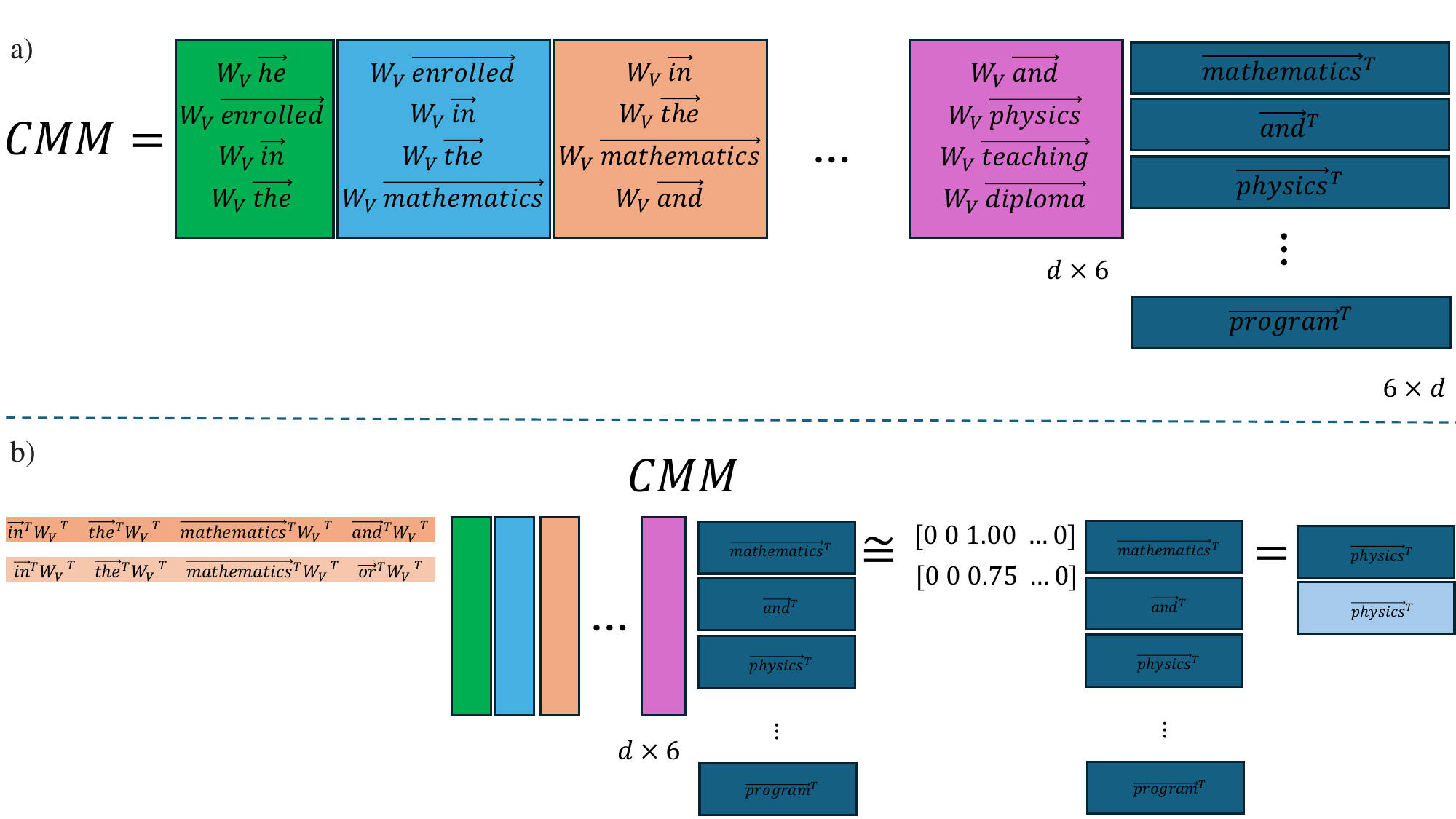}
  \caption{A sample Language Model (LM) with a single Correlation Matrix Memory (CMM) coding a single sentence. a) Memorization phase: the CMM is a $d \times d$ matrix coding the pairs (sequence, next\_token) for a sentence; b) Retrieving phase: a sample use of the CMM in (a) where the CMM emits the vector of the word \emph{physics} given the encoding of the sequence \emph{in the mathematics and}.}
  \label{fig:CMM-LM-single}
\end{figure}

The translation of these multi-sets $LM$ in a CMM is straightforward: input sequences $s$ are keys, and output next tokens $y$ are values. We then use multivariate Gaussian vectors stored in the matrix $E_{n \times d}$ to encode the $n$ tokens in $V$ and a Johnson-Lindestrauss Transform $W_V$ to ensure that both input sequences and output vectors are in the same space $R^d$. Then, the CMM encoding  an LM has the following equation:
\begin{equation}
C = \sum_{(s,y) \in LM} \vec{s} \vec{y}^T = \sum_{(s,y) \in LM} \left[\begin{array}{c}W_V\vec{x}_1\\W_V\vec{x}_2\\\vdots\\W_V\vec{x}_h\end{array}\right]\vec{y}^T
\label{eq:sequence_el_dot_prod}
\end{equation}
where $\vec{s} \in R^d$ is the vector representing the sequence $s$ composed as described using vectors $\vec{x}_i \in R^d$ encoding tokens $x_i$ and the JLT matrix $W_V$ of dimensions $d/h \times d$. The vector $\vec{y} \in R^d$ represents the symbol $y$. Vectors $\vec{x}_i$ and $\vec{y}$ are columns of the embedding matrix $E$. 
The properties of the embedding vectors and the JLT, along with how the JLT is built, can guarantee that:
$$
(W_V\vec{x}_j)^T W_V\vec{x}_i  \approx \left\{\begin{array}{ll} 1/h & \text{ if } x_i = x_j \\0 & \text{ if } x_i \neq x_j  \end{array} \right.
$$

Once the LM is transferred to the CMM, the matrix $C$ can be used to predict the next token of a given sequence $\hat{s} = [\hat{x}_1,\hat{x}_2,...,\hat{x}_h]$. The next token can be derived as follows. The first step is the product: 
\begin{equation}
\vec{\hat{y}} = \vec{\hat{s}}^T C = \sum_{(s_j,y_j) \in LM} (\vec{s}^T\vec{s_j}) \vec{y_j}
\label{eq:dec_singleCMM}
\end{equation}
where  $\vec{s}^T = [\vec{\hat{x}}_1^T W_V^T, \vec{\hat{x}}_2^T W_V^T,...,\vec{\hat{x}}_h^T W_V^T]$ is the representation in a space $R^d$ of the sequence $s$. The above properties (see eq.~\ref{eq:sequence_el_dot_prod}) guarantee that: 
$$
\vec{s}\vec{s}_i^T \approx k/h
$$
where $k$ is the number of common tokens between the sequences $s$ and $s_j$. Indeed, the CMM transformation of the LM also offers an initial property of generalization. The models can also give an estimation of the count for sequences that are not stored completely. 
Therefore, the following product estimates the counts of an output token $t_i$ given the sequence $\hat{s}$:
$$
\vec{t} = E \vec{\hat{y}} 
$$
Hence, focusing on the $i$-th component of the vector $\vec{t}$, it will be the approximate count of full and partial sequences generating the i-th token, that is:
$$
(\vec{t})_i \approx \sum_{\{(s_j,y_j) \in LM | y_j=t_i \}} \vec{s}^T\vec{s_j} 
$$
The token $t_i$ to emit for a sequence $\hat{s}$ is then chosen by selecting the index $i$ of the component of the vector $E \vec{\hat{s}}^T C$ with the highest value as in this equation:
\begin{equation}
    i = argmax_i ( E \vec{\hat{s}}^T C )_i
\label{eq:argmaxCMM}
\end{equation}

To better describe how a simple correlation matrix memory (CMM) can be used as a language model (LM), we show how to build an LM with a window of 4 tokens using the following sentence as a running example:
\begin{center}
\begin{tabular}{p{7cm}}
\emph{He enrolled in the mathematics and physics teaching diploma program}\\
\end{tabular}
\end{center}
Then, the CMM should contain the set $LM$ of pairs: 
\begin{center}
\begin{tabular}{p{7cm}}
$LM$ = \textit{\{([He enrolled in the],mathematics), ([enrolled in the mathematics], and), ([in the mathematics and], physics), ..., ([and physics teaching diploma],program)\}}\\
\end{tabular}
\end{center}
Hence, given a $d$-dimensional word embedding space where vectors $\vec{w}$ for each word $w$ are drawn from a Gaussian multinomial pseudo-random generator and $W_V$ is a Johson-Lindestrauss Transform ${d \times d/4}$ matrix embedding word vectors in a smaller space $R^{d/4}$, the CMM $d \times d$ matrix will contain the sum of the matrices representing the pairs in $P$ (see Fig. \ref{fig:CMM-LM-single}.a) built as the sum of outer products of key columns representing sequences and row value vectors representing next tokens. For example, the first green column represents the sequence \emph{He enrolled in the} and it is linked with the first row representing \textit{mathematics} (see Fig. \ref{fig:CMM-LM-single}.a). 

In the retrieving phase, to obtain the next token given a sequence of 4 tokens, the transposed vector representing the sequence is multiplied by the CMM. The result is the vector representing the next token. For example, given the sequence \emph{in the mathematics and}, the green transposed vector representing the sequence is multiplied by the CMM representing encoded associations (see Fig. \ref{fig:CMM-LM-single}.b). The multiplication of this vector with the first block implied by the CMM produces a vector that approximates $[\begin{array}{cccccc} 0&0&1.00&0&0&0\end{array}]$. This vector then extracts the third vector of the second block, that is, the one associated with \textit{physics}. This model can also be generalized in the sense that it may take into consideration subsequences of a given sequence. Indeed, the sequence \emph{in the mathematics or} will emit the vector for $physics$ with a weight of $0.75$ given the value of the dot product of its vector with the vector of the sequence \emph{in the mathematics and}. This is the first possible generalization of the one-layer language model built with a CMM.

Hence, a single CMM can build language models able to generalize but these language models will operate with fixed small windows depending on the ratio $d/h$, dimension of the space with respect to the number of heads or tokens in the window. If $d/h$ is small, vectors in this smaller space will be not enough different to discriminate different tokens.  

\subsection{Multi-layer Correlation Matrix Memories}
\label{sec:multi-layerCMM}

\begin{figure*}
  \includegraphics[width=\textwidth]{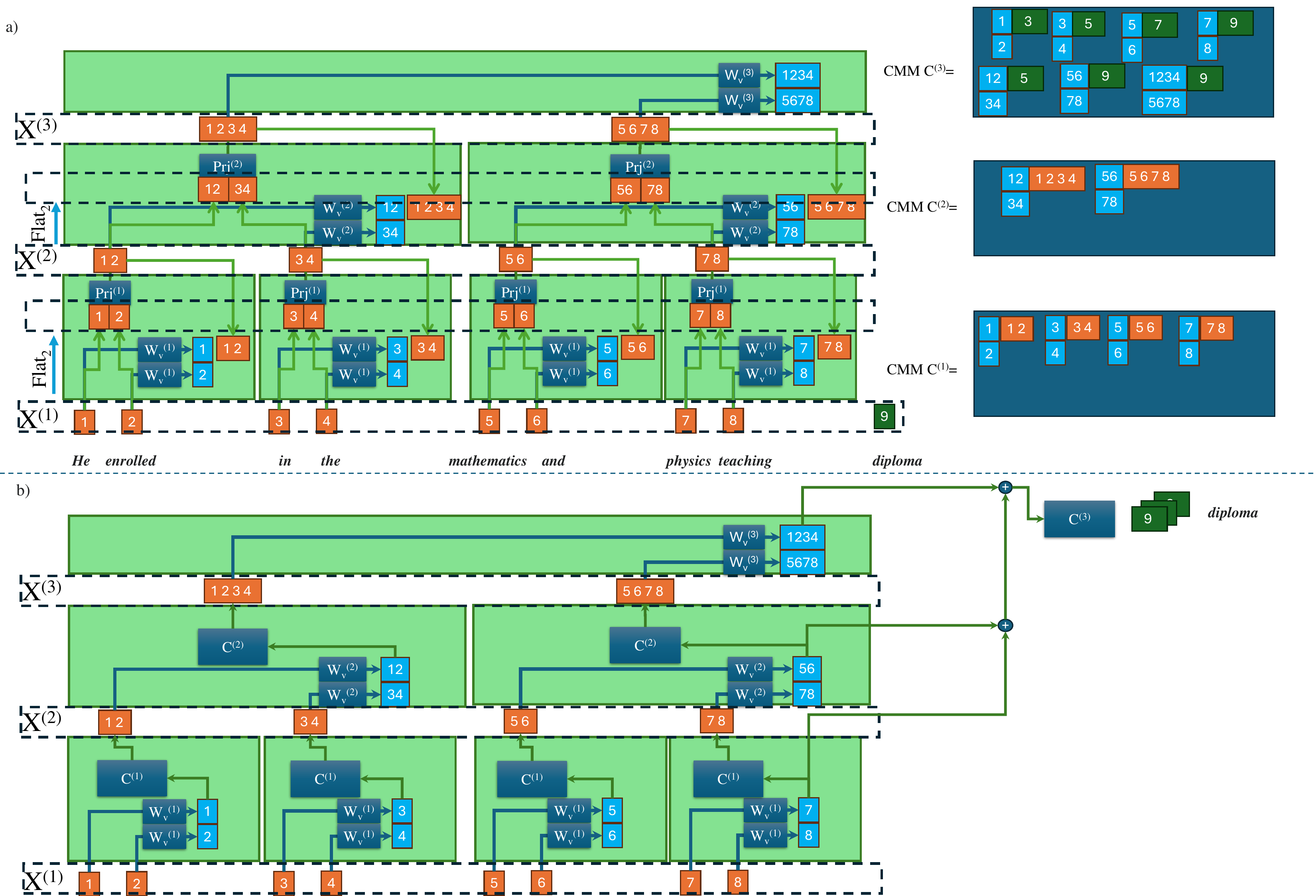}
  \caption{A sample Language Model (LM) with a Multi-layer Correlation Matrix Memory (CMM) coding a sequence of numbers with number of heads h=2 and number of layers l=3.}
  \label{fig:stack}
\end{figure*}

To increase the maximum length of the input window of language models, in line with what is done in transformers \cite{10.5555/3295222.3295349}, we stack layers containing correlation matrix memories (see Fig. \ref{fig:stack} for an example). 

The driving idea is that CMMs of a generic MeMo layer store the encoding of sequences whose length is determined by the level of the layer. Hence, the generic MeMo layer contains key-value pairs where the key is the representation of the sequence elements, and the value is a vector representing the sequence as a whole. The representation of the sequence elements is done similarly to what is done for an LM based on a single CMM (as in Sec. \ref{sec:single_LM-CMM}).  The last MeMo layer instead stores the relation between sequences of increasing length and the next token, and, thus, it is the layer devoted to the next token prediction.

To define MeMo, we need first to fix the notation: $h$ is the number of heads or, also, the maximum number of input elements that are treated by the MeMo layer, $l$ is the number of layers, $d$ is the dimension of the encoding vectors, and $X^{(i)}$ is the input for the $i$-th layer containing vectors representing sequences in row vectors $\vec{x}^{(i)T}_j$. 
Given these parameters, MeMo can encode sequences of a maximum length of $m=h^l$. 

\paragraph{Memorization} Each MeMo layer $MM^{(i)}$ \textit{memorizes} sequences up to the length $h^i$ and produces the next token emission matrices for sequences up to $h^i$ length to be stored in the last layer. The equations for the memorization phase are the following:
$$
MM_{m}^{(i)}
%= 
\left\{
\begin{array}{lcl}
X^{(i+1)} & = & Flat_h(X^{(i)}) Prj^{(i)}\\
I^{(i)} & = & Flat_h(X^{(i)}W_V^{(i)T})\\
C'^{(i)} & = & C^{(i)} + I^{(i)T} \Phi^{(i)} X^{(i+1)}\\
C'^{(last)} & = & C^{(last)} + I^{(i)T} Sel_h(X^{(1)})
\end{array}
\right.
$$
where $Flat_h(X^{(i)})$ is a function that takes a $k \times d$ matrix  and reshapes it in a $k/h \times d \cdot h$ matrix, $Sel_h(X^{(0)})$ is a function that selects every h vector from the input matrix $X^{(0)}$, $Prj^{(i)}$ is a $h\cdot d \times d$ projection matrix that encodes sequences of $h$ vectors in the internal $d$ dimensional space, and $W_h^{(i)}$ is an embedding matrix reducing vectors in $R^d$ to vectors in $R^{d/h}$. 

We proceed by reading the equations from the top to the bottom. 

Each $h$ vectors in the input $X^{(i)}$ are juxtaposed to create sequences of input that are treated by each block of the $i$-th layer and, thus, these sequences of inputs are encoded as in vectors $X^{(i+1)}$ of dimension $d$ that are unique for each encoded sequence. 

Sequences are also represented by vectors $I^{(i)}$ by first embedding vectors $X^{(i)}$ in sequences $X^{(i)}W_V^{(i)T}$ of row vectors in $d/h$ and, then, packing these vectors in single row vectors $Flat_h(X^{(i)}W_V^{(i)T})$ representing sequences. These $I^{(i)}$ are the keys of sequences, and $X^{(i+1)}$ are the values in which these keys are translated in the retrieving phase. 

Then, $I^{(i)}$ are intended to represent sequences as sequences of elements $\vec{x}_j^{(i)T} W_V^{(i)T}$. Instead, $X^{(i+1)}$ represents the same sequences as a whole. This difference is small but important as $I^{(i)}$ are intended to be also partially matched. 

The pairs  (sequences of elements, coding of sequence), respectively in $I^{(i)}$ and $X^{(i+1)}$, are then stored in the CMM $C^{(i)}$ of the current level $i$ adding $I^{(i)T} \Phi^{(i)} X^{(i+1)}$ to the current matrix.
The diagonal matrix $\Phi^{(i)}$ contains penalizing factors to force only one memorization of the pair (sequences of elements, coding of sequence) in the corresponding matrix $C^{(i)}$. 
The pair (sequences of elements, coding of sequence) should be stored if it is not stored in the current matrix $C^{(i)}$, and if it appears $f$ times in the current updated, it should be stored only once. Therefore, the penalizing matrix $\Phi^{(i)}$ is the product of two diagonal matrices:
$$
\Phi^{(i)} = D^{(i)} F^{(i)}
$$
where: (1) the distiller $D^{(i)}$ is a filter of patterns and has 0 in the diagonal if the corresponding pattern is already stored in $C^{(i)}$ and 1 if it is not stored in $C^{(i)}$; (2) the inverse frequency matrix $F^{(i)}$ is the diagonal of $F^{(i)}$ where elements in the diagonal contains the inverse frequency of the corresponding pattern in the current update $X^{(i+1)}$. The two matrices $D^{(i)}$ and $F^{(i)}$ are obtained with linear and nonlinear operations over the current matrices of the current layer. Given $\overline{x}^{(i+1)}$ as the sum of all the row vectors in $X^{(i+1)}$, the distill matrix is computed as follows:
$$
D^{(i)} = diag(1 - round(I^{(i)} C^{(i)} \overline{x}^{(i+1)}))
$$
where $I^{(i)} C^{(i)}$ produces all sequence vectors already stored in $C^{(i)}$ and, then, the multiplication with the vector $\overline{x}^{(i+1)})$ detects which of these vectors is in the new vectors to store. The frequency matrix is computed similarly:
$$
F^{(i)} = diag(1/round(X^{(i+1)} \overline{x}^{(i+1)}))
$$
by multiplying the same vector $\overline{x}^{(i+1)}$ with all the vectors to be stored. 

Finally, in each layer $i$, the CMM $C'^{(last)}$ of the last layer is updated with the pairs connecting the sequences of elements $I^{(i)T}$ with the correlated next tokens $Sel_h{X^{(1)}}$. The last layer is the real layer that emits the next token of a given sequence. 

We show how the memorization of the simple sequence 1 2 3 4 5 6 7 8 9 representing the sentence of the running example is done in a MeMo with $h=2$ and $l=3$ (see Fig. \ref{fig:stack}.a). This configuration of MeMo allows the storage sequences of up to 8 tokens, emitting the ninth token. In this example, the CMM $C^{(1)}$ of layer $1$ is storing the coding of sequences of two input elements. Embedding vectors of dimension $d$ are represented in orange and embedding vectors of dimension $d/2$ are represented in light blue. Sequences $I^{(i)}$ of elements are the light blue vector pairs 1 2, 3 4, 5 6, and 7 8. These are multiplied with the coding of the sequences represented by the orange vectors 12, 34, 56, and 78. These outer products are stored in CMM $C^{(1)}$. Instead, the outer product of vectors 1 2, 3 4, 5 6, and 7 8 with the vectors 3, 5, 7, and 9 is stored in the matrix CMM $C^{(3)}$. By using embeddings  $X^{(2)}$ of layer 1, layer 2 emits the embeddings of length four and stores them in the matrix $C^{(3)}$. Then it store the pairs ([1 2, 3 4], 5) and ([5 6, 7 8], 9) ih $C^{(3)}$. Layer 3 stores the pair ([1 2 3 4, 5 6 7 8], 9) in $C^{(3)}$ that represents the longest sequence that can be stored given $h$ and $l$.       

\paragraph{Retrieving} 
In this phase, MeMo is used to retrieve what has been stored by giving as input a sequence and expecting the next token as output. All intermediate layers are used to retrieve the encoding of sequences with growing length. These are used on the final layer to retrieve the next token to emit. The retrieving equations for each layer of MeMo are the following:    
$$
MM_{r}^{(i)}
%= 
\left\{
\begin{array}{lcl}
I^{(i)} & = & Flat_h(\hat{X}^{(i)}W_V^{(i)T})\\
\hat{X}^{(i+1)} & = & I^{(i)T}C^{(i)}\\
O'^{(last)} & = & O^{(last)} + I^{(i)T}C^{(last)}
\end{array}
\right.
$$
where $\hat{X}^{(i+1)}$ are the retrieved encoding of the sequences extracted from the CMM $C^{(i)}$ of the current layer by using the encoding of the sequences of elements $I^{(i)}$. Clearly, $\hat{X}^{(1)} = X^{(1)}$, that is, the first layer encodes the sequence as it is, and it is not retrieved from a CMM. Finally, $O^{(last)}$ is storing the output vectors for the next token given the input sequence.

In the running example, the retrieving is done as follows (see Fig. \ref{fig:stack}.b). The sequence 1 2 3 4 5 6 7 8 is used to generate the first sequence of vectors $X^{(1)}$. Each pair is used to generate the encoding of sequences of elements (light blue boxes) by using the matrix $W_v^{(1)}$. Then, these are used to retrieve the encoding of sequences from $C^{(1)}$; the encoding is the light orange boxes. The encoding $E_1$ of the sequence of elements of the last part of the sequence 7 8 is summed up to then retrieve the next token from $C^{(3)}$. The following level works in the same way, emitting the encodings $E_2$ and $E_3$ of the sequences of elements 56 78 for layer 2 and 1234 5678 for layer 3, respectively. The sum $E_1+E_2+E_3$ of three emitted encodings is then used to retrieve the next token by multiplying the resultant vector with the matrix $C^{(3)}$. Then, the result will be the embedding vector of 9 with a weight of 3 since it is encoded three times in the matrix with three different sequences of elements.         

\paragraph{Forgetting}
MeMo, as it is, offers then the important capability of forgetting, that is, erasing stored sequences. The operation is straightforward: subtracting the sequence from the last layer instead of summing. The equation follows:
$$
MM_{f}^{(i)}\\
%= 
\left\{
\begin{array}{lcl}
X^{(i+1)} & = & Flat_h(X^{(i)}) Prj^{(i)}\\
I^{(i)} & = & Flat_h(X^{(i)}W_V^{(i)T})\\
C'^{(last)} & = & C^{(last)} - I^{(i)T} Sel_h(X^{(1)})
\end{array}
\right.
$$

\section{Experimental Investigation}

In this section, we experiment the memorization capacity of MeMo with a single layer and with multiple layers. 

\subsection{Exploring Memorization Capabilities of Single-layer MeMo}

\paragraph{Experimental set-up}
In the first experiment, we investigate the capacity of a single-layer MeMo to memorize the association between sequences of symbols and one output symbol. Hence, we created a generator of random sequences of $h$ symbols $[x_1,x_2,...,x_h]$ that are mapped to a random symbol $y$. To maximize the diversity, symbols are taken with a uniform random distribution from a vocabulary of 100,000 symbols. This guarantees that the mapping between sequences and symbols is unique. Therefore, we are testing the real capacity of memorization of the CMM. In the experiments, we used random vectors $\vec{x_i}$ representing symbols $x_i$ with $d$ dimensions with $d_h \in \{16, 32, 64, 128, 256\}$ and we experimented with sequences of increasing length with $h \in \{2, 4, 8, 16, 32\}$. The output vectors  $\vec{y}$ representing symbols $y$ are instead random vectors with $d$ in $\{512, 1024, 2048, 4096, 8192\}$. Therefore, experimental CMMs are matrices with $(h \times d_h, d)$ dimensions. Thus, the number of parameters of each CMM is $NoP = h \cdot d_h \cdot d$. 

\begin{figure}
  \includegraphics[width=0.5\textwidth]{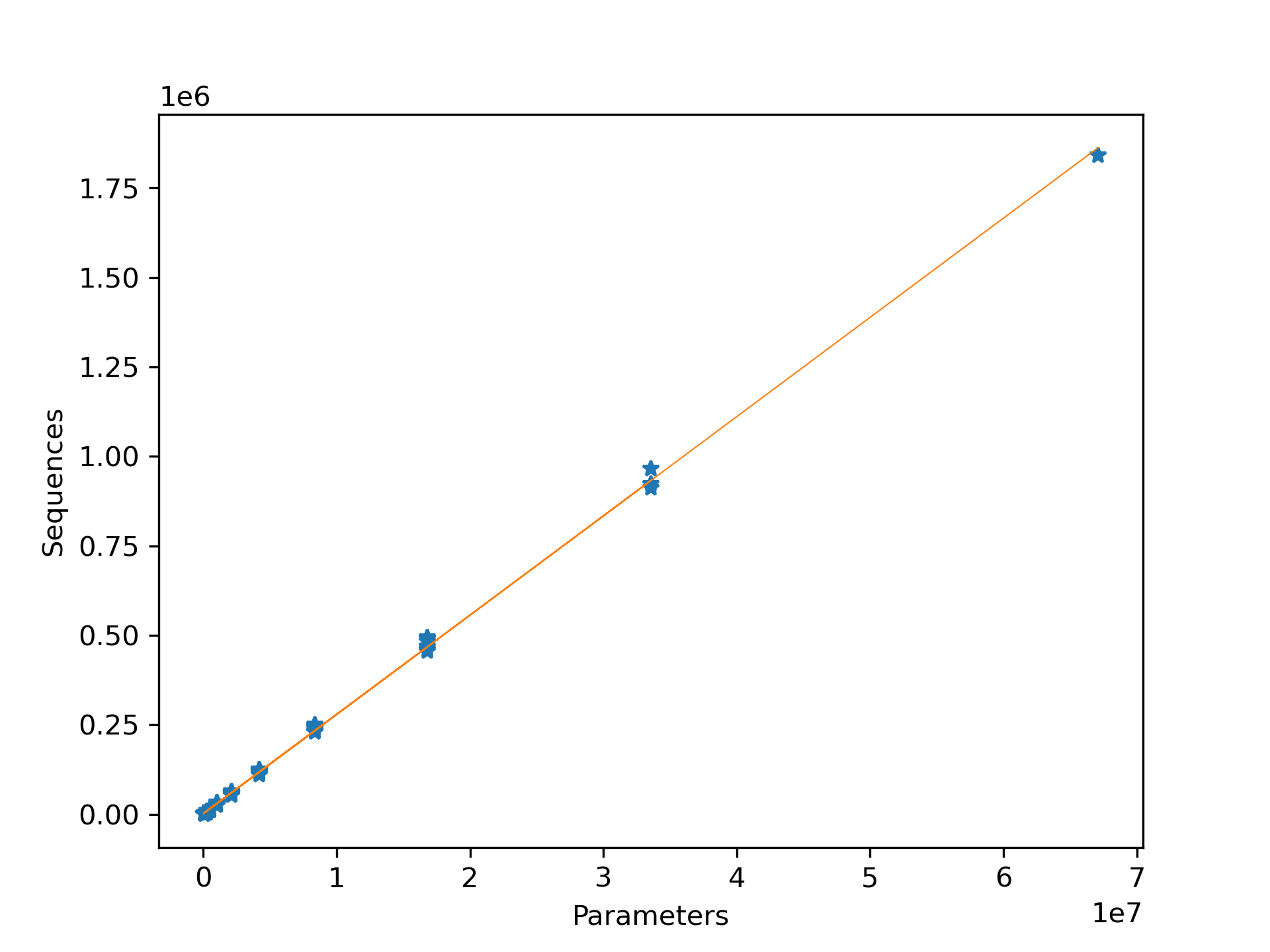}
  \caption{Memorization capacity of a single CMM: parameters $NoP = h \cdot d_h \cdot d$ with respect to the number of sequences that can be stored. Points in the plot are CMMs with different configurations of $h$, $d_h$, and $d$.}
  \vspace{-1em}
  \label{fig:cmm_storingcapacity}
\end{figure}

In this experiment, batches $B_i$ of 1,000 pairs $\{([x_1,x_2,...,x_h],y)\}$ are stored into the CMM matrix $C$ for each step $i$ and, then, the storing capacity is evaluated by computing the accuracy of reproducing the tokens of the batch $B_i$ and the first batch $B_0$. The accuracy $Acc(B_i, C)$ of the CMM $C$ on the batch $B_i$ is computed as the percentage of correct emitted tokens $y$ given sequences $[x_1,x_2,...,x_h]$ with equation \ref{eq:argmaxCMM}. The storing capacity of a CMM matrix C is computed as the number of pairs that can be stored that guarantee an $(Acc(B_0, C) + Acc(B_i, C))/2>0.9$ where $B_0$ is the first batch and $B_i$ is the current batch. 

\paragraph{Results}
Memo, based on a single correlation matrix memory, has the capacity to store sequences according to the total number of parameters of the CMM. Indeed, the memorization capacity of a single CMM does not depend on the number of heads of the input sequence but only on the total number of parameters of the CMM. The plot in Figure \ref{fig:cmm_storingcapacity} reports the results of the first set of experiments and shows that there is a linear relation between the number of parameters and the number of stored sequences. This is in line with the empirical findings on LLMs that originated the linear scaling law linking the number of tokens of the training corpus with respect to the total number of parameters of the Transformer \cite{kaplan2020scalinglawsneurallanguage}.

\begin{figure*}
\begin{tabular}{ccc}
  \includegraphics[width=0.3\textwidth]{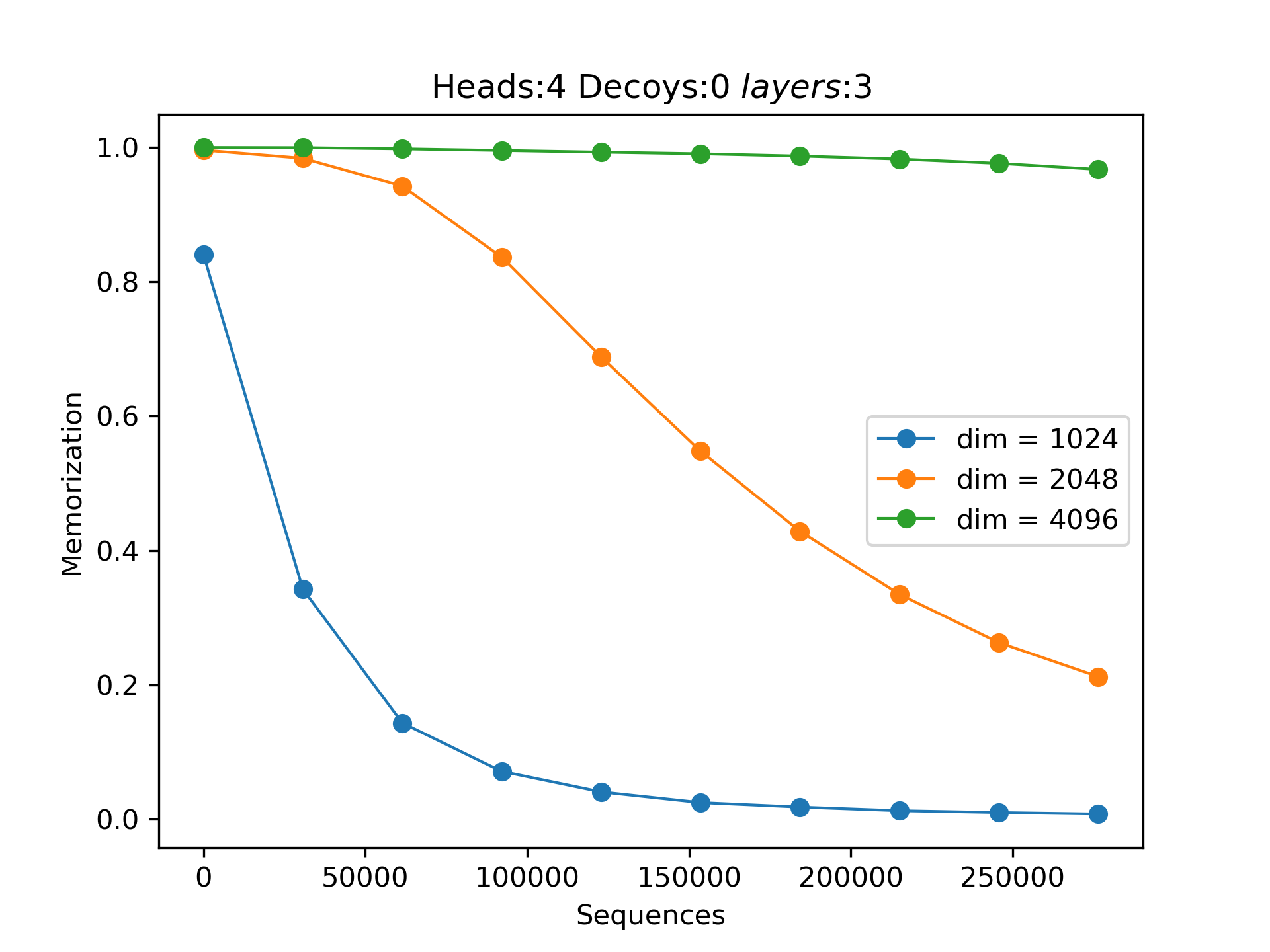}   & 
  \includegraphics[width=0.3\textwidth]{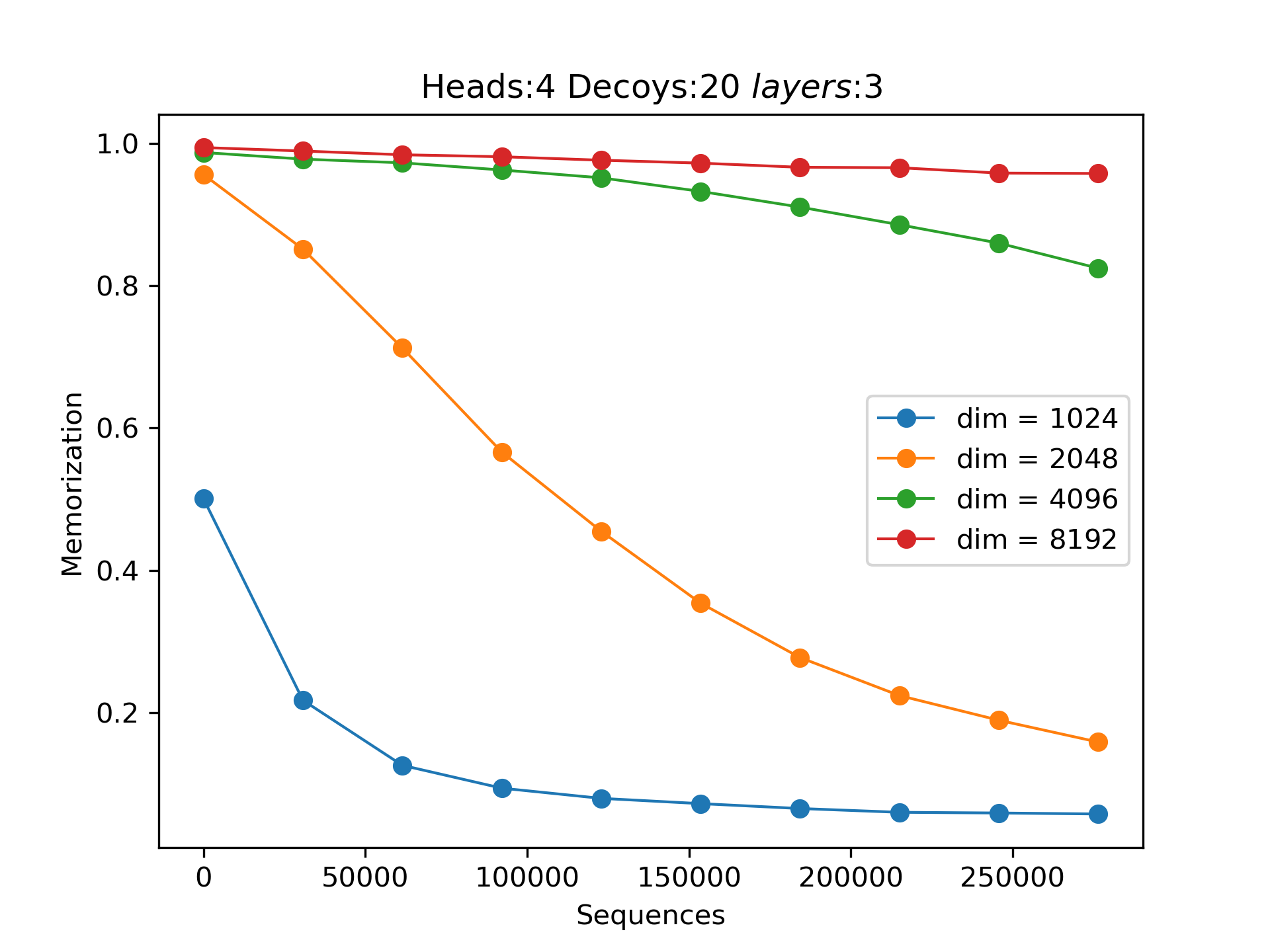}   & 
  \includegraphics[width=0.3\textwidth]{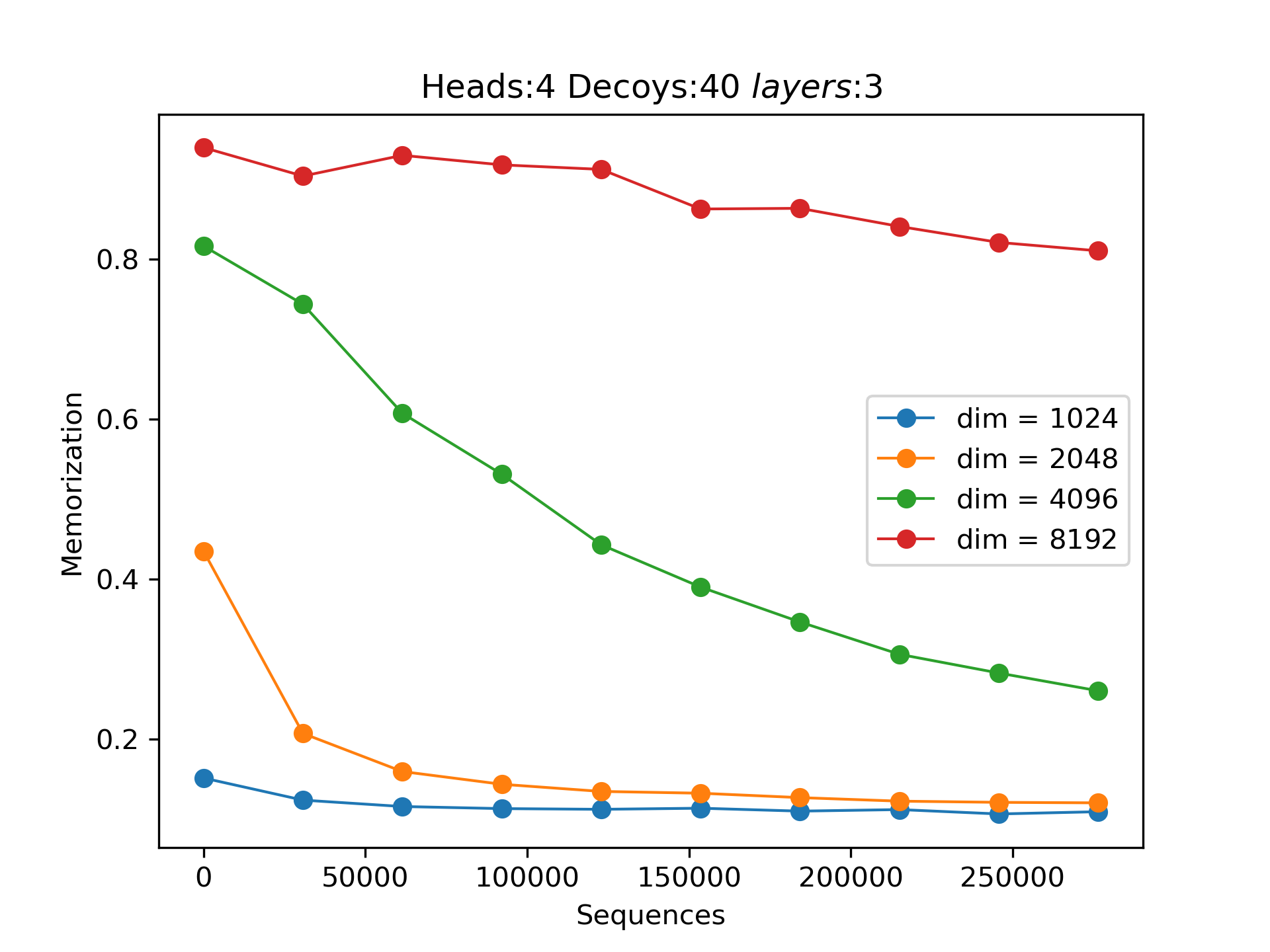}   \\
  \includegraphics[width=0.3\textwidth]{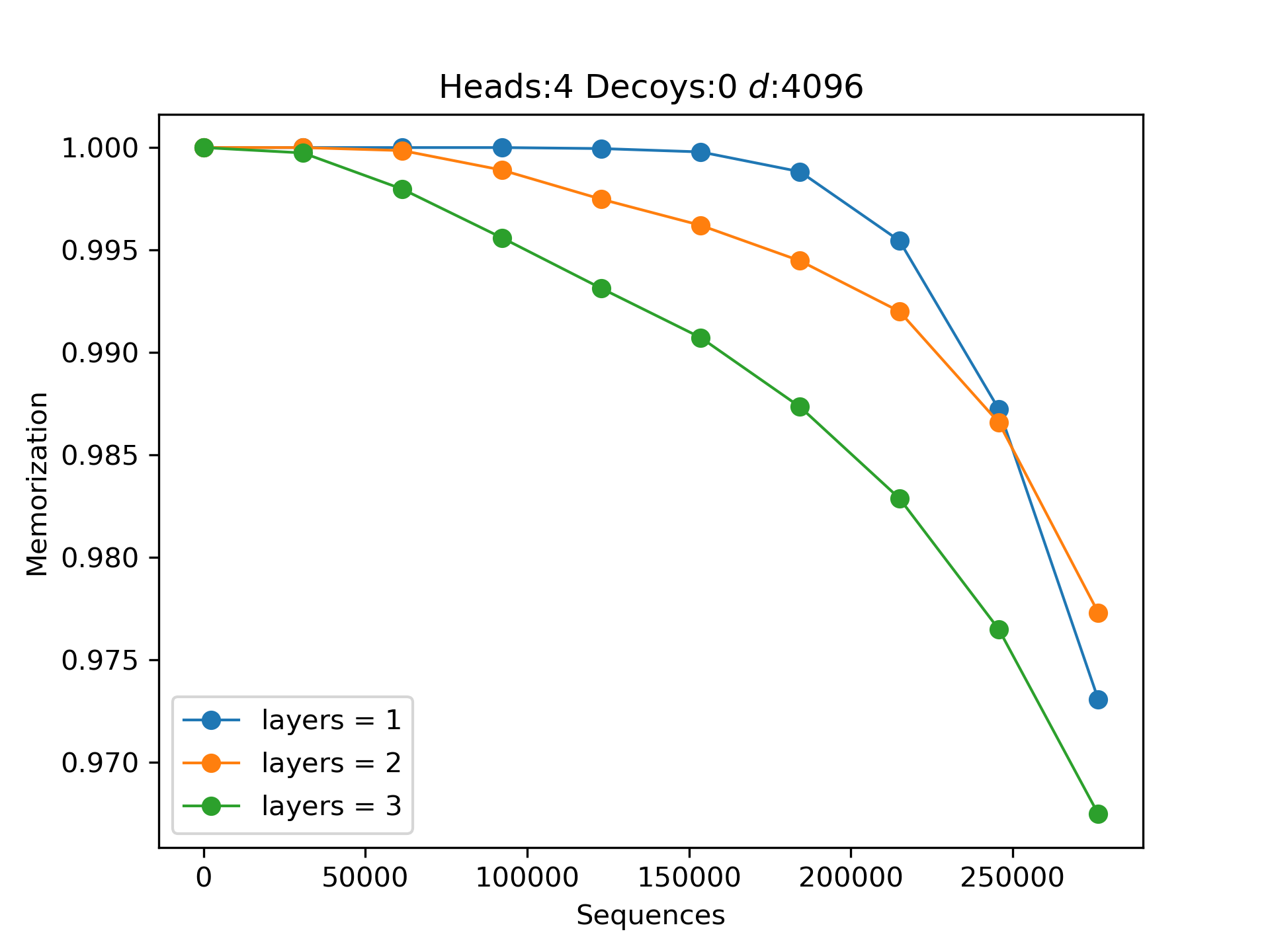}   & 
  \includegraphics[width=0.3\textwidth]{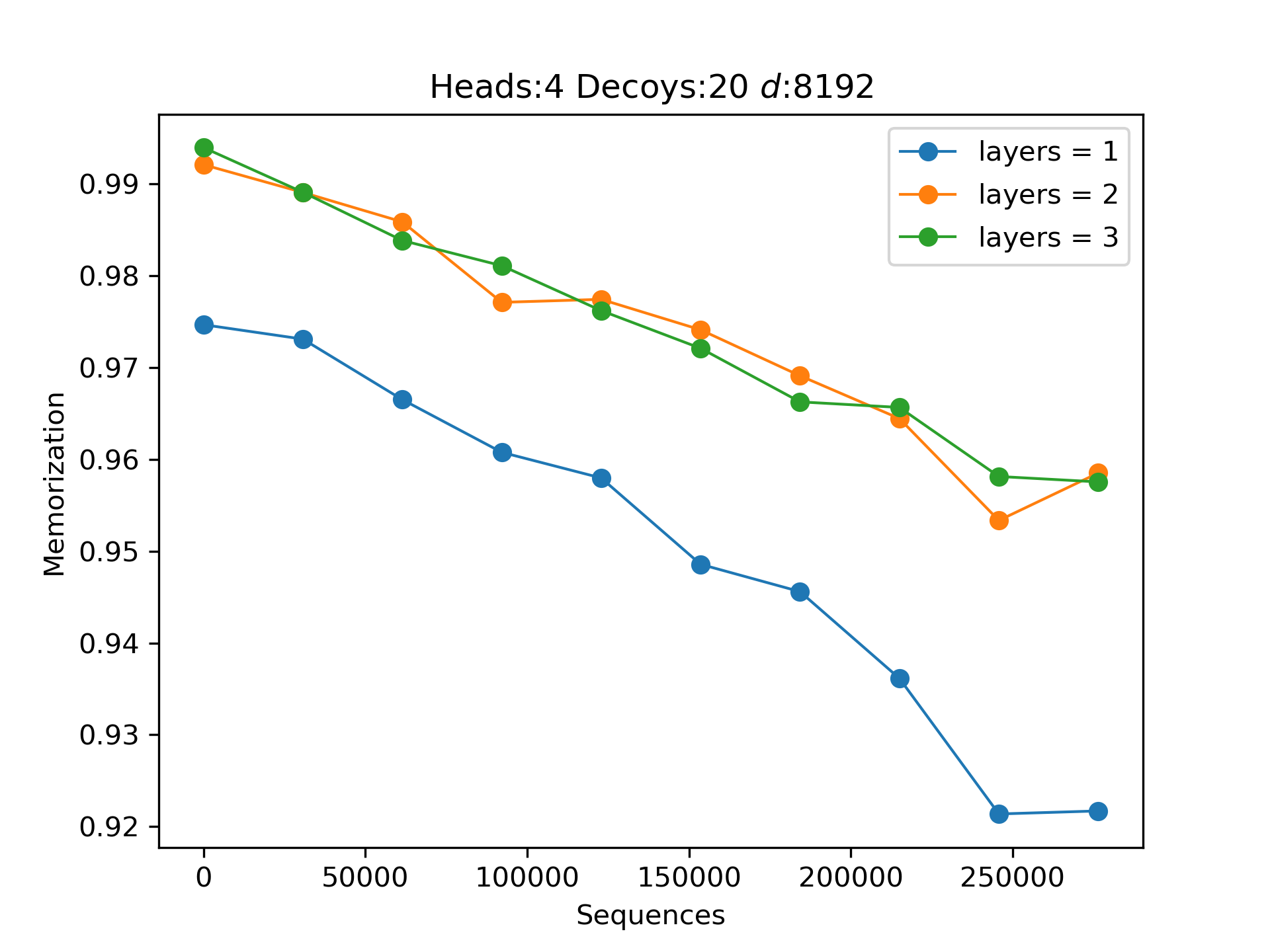}   & 
  \includegraphics[width=0.3\textwidth]{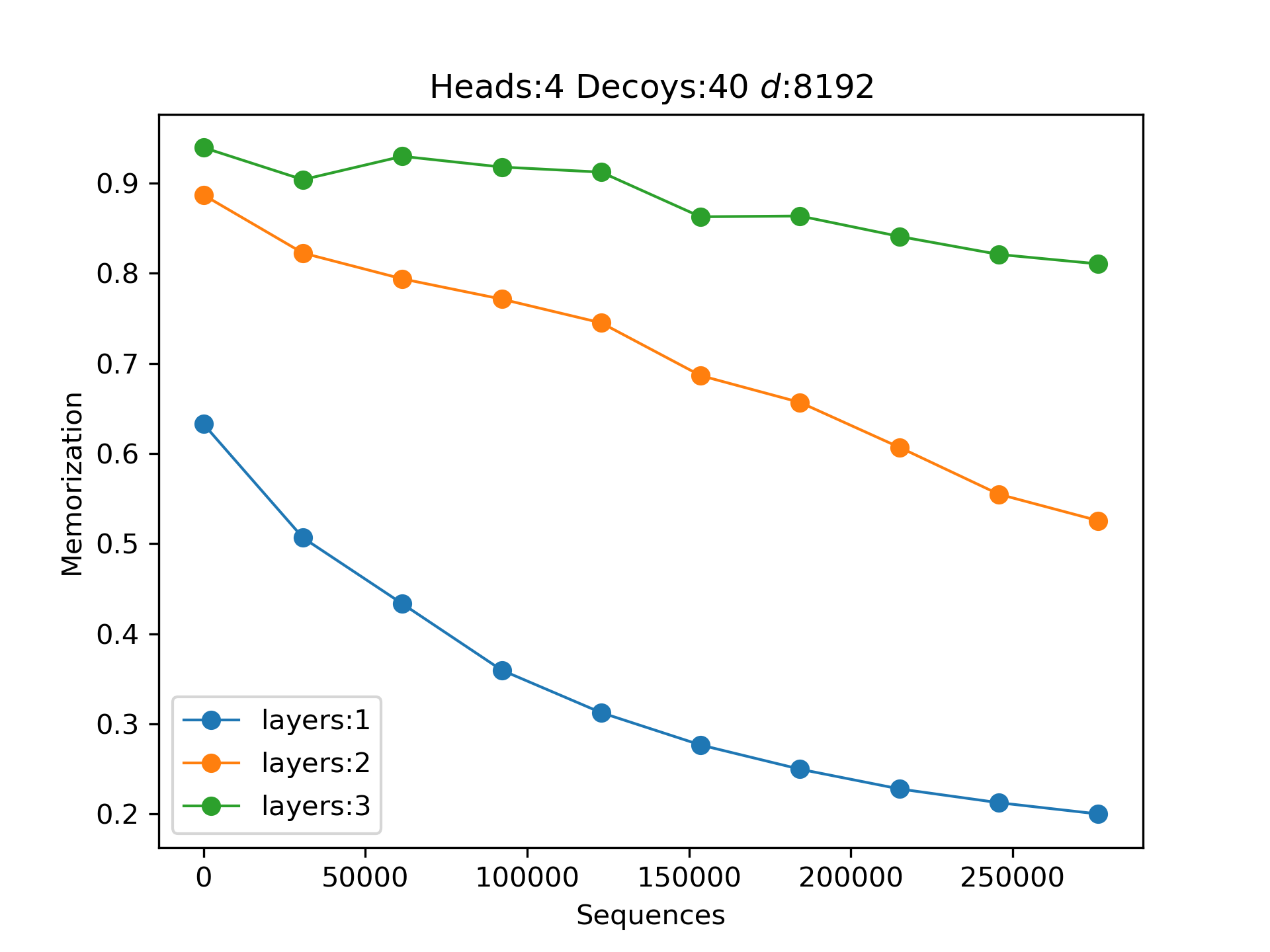}   \\
\end{tabular}
  \vspace{-.5em}
  \caption{Memorization capacity of MeMo: storing ability with respect to the number of stored sequences. Experiments with increasing complexity of the datasets (number of decoys) and increasing number of layers} 
  \vspace{-1em}
  \label{fig:memo_storingcapacity}
\end{figure*}

\subsection{Exploring Memorization Capabilities of Multi-layer MeMo}

\paragraph{Experimental set-up}
In the second experiment, we investigate the capacity of MeMo to memorize complete texts. As we aim to investigate only the memorization capacity, we used randomly generated texts of a given chunk length. To really test the capacity of splitting the ability to store long sequences with a layered model, we produced a text generator that simulates the existence of repeated words long $h$ tokens in the text. These repeated words decoy a memorizer with only $h$ heads because 
the same $decoy$ of $h$ tokens should produce different next tokens according to the tokens preceding the $decoy$, which may be captured only if MeMo with more layers is memorizing sequences longer than $h$. We experimented with $h=4$, with up to $3$ layers, with $d \in \{1024, 2048, 4096, 8192\}$, and with three setting of decoys: $0$, $20$, and $40$. 

\paragraph{Results}

The memorization capability of MeMo increases with the inner dimension $d$ that is correlated with the total number of parameters. In the three cases with the three different levels of decoys, the memorization capability of texts increases with the inner dimension for MeMo with 3 layers (top line of plots in Fig. \ref{fig:memo_storingcapacity}). As the dimension of the representation of elements of the sequence of tokens is $d/4$, the capability of storing sequences strongly depends on $d$. Hence, to obtain a reasonable degree of memorization, an internal representation of at least $d=4096$ is needed. Indeed, only with $d=4096$, the performance of the MeMo with three layers on the memorization of completely different sequences (decoys=0) stays constantly over $0.97$. When the complexity of sentences increases, a larger $d$ is needed. A sufficient level of memorization is guaranteed with $d=8192$ when decoys are 40. Overall, increasing the inner dimension $d$ enables better memorization.

As expected, augmenting the number of layers increases the ability to memorize. For the three levels of decoys, increasing the number of layers has a positive effect on the memorization performance (see bottom of Fig. \ref{fig:cmm_storingcapacity}). 
Indeed, as the complexity increases, that is, as the number of decoys increases, the importance of having more layers becomes clearer. With 20 decoys, at least two layers are needed. With two or three layers, the storing capacity is above $0.96$ for at least 250,000 sequences.  Whereas, with 40 decoys, at least three layers are required to have a storing capacity of more than $0.88$ for 250,000 sequences. 

Results show that MeMo with multiple layers can expand the memorization capacity of MeMo with single layer and, thus, open the possibility to create transparent language models.

\section{Conclusion and Future Work}

Memorization is a key component of transformer-based Large Language Models. Hence, in this paper, we proposed to shift the paradigm by designing language models based on memorization. We then presented MeMo as a novel way to build language models using correlation matrix memories stacked in layers. Experimental evaluation has shown that MeMo-like architecture can memorize sequences of tokens. 

By using memorization, MeMo-like architectures are transparent and editable by design and opens back the possibility to include explicit knowledge modeling in neural network language models. Indeed, MeMo can help leverage traditional linguistic studies in this era, where transformer-based large language models are obtaining unprecedented performance. With MeMo, we could control how linguistic knowledge is used to generalize examples, we could embed transformation rules, and we could represent knowledge graphs and linguistic ontologies. In other words, MeMo gives back control to knowledge experts, linguists, and NLP practitioners with the aim of reducing data hungriness of Large Language Models.

\section*{Limitations}
The approach proposed in this paper is a paradigm shift, and then, the software implementing the model has some compatibility issues with the existing software ecosystem of transformers in Hugging Face. Hence, it has not been possible to experiment with the model using the current evaluation suites. Although this is a limit with respect to the comparability of MeMo with current transformer-based LLMs, it does not represent a major limit concerning the memorization capability of MeMo. 

\section*{Ethical Statement}

Making memorization more evident and being editable by design, MeMo may allow an easier control of the stored texts by mitigating leaks of sensible data and social biases.   

\bibliography{references}

\newpage
\appendix

\onecolumn

\section*{Appendix A: Analyzing Storing Capacity of Random Vectors}

This section explores theoretically how many nearly orthogonal unit vectors can be stored in a set $NOV(\varepsilon,\theta)$ in the space $R^d$, where $\varepsilon$ is the approximation required and $1 -\theta$ is the probability that this approximation is guaranteed. For two vectors $\vec{a}$ and $\vec{b}$ in $NOV(\varepsilon,\theta)$, the following should hold: 
\begin{equation}
P(\vec{e_a}\vec{e_b}^{\top} - \varepsilon \leq \vec{a}\vec{b}^\top \leq \vec{e_a}\vec{e_b}^\top + \varepsilon) \geq 1- \theta
\label{eq:nov}
\end{equation}
In other terms, if $a$ and $b$ are the same generalized sequence, $\vec{a}\vec{b}^\top \approx 1$, whereas, if if $a$ and $b$ are two different generalized sequences, $\vec{a}\vec{b}^\top \approx 0$. 
There is a long-lasting conjecture that postulates a relation between $d$ and $m$ for any given $\theta$ and $\varepsilon$ \cite{HechtNielsen94} but, to the best of our knowledge, a definitive demonstration does not still exist. By using the Johnson\&Lindestrauss Lemma \cite{JLL}, we derived an upper-bound for $d$. 
Sets $NOV(\varepsilon,\theta)$ can potentially host\footnote{The expression \emph{The set $NOV(\varepsilon,\theta)$ can potentially host ...} stands for the more formal \emph{There is a probability strictly greater than 0 that $NOV(\varepsilon,\theta)$ contains ...}} $m$ vectors with $\theta = 2/m^2 - 1/m^4$ according to this relation:

$$m \leq e^{8(\varepsilon^{2} - 4/3\varepsilon^{3})d}$$ 
Thus, there is an exponential relation between $d$ and $m$. This is a positive result as spaces $\R^d$ can host large sets of $NOV(\varepsilon,\theta)$.

Thus, definitely many substructures in $S$ in real datasets can be represented with vectors in $NOV(\varepsilon,\theta)$.

\paragraph{Existing results}
Our corollary stems from two results \cite{JLL,JLLsimple_demonstration}:

\begin{theorem}[Johnson-Lindenstrauss Lemma]
\label{th:JJL} For any $0 < \epsilon < 1$ and any integer $m$. Let $d$ be a positive integer such that

$$d \geq 4(\epsilon^{2}/2 - \epsilon^{3}/3)^{-1}\ln{m}$$

Then for any set $V$ of $m$ points in ${\mathbb R}^k$, there is a map $f : {\mathbb R}^k \rightarrow {\mathbb R}^d$ such that for all $\vec{u}, \vec{v} \in V$, 
\begin{center}
$(1- \epsilon) \|\vec{u}-\vec{v}\|^2_2 \leq \|f(\vec{u})-f(\vec{v})\|^2_2 \leq  (1 + \epsilon)\|\vec{u}-\vec{v}\|^2_2 $.
\end{center}
\end{theorem}

$$m \leq e^{\frac{(\epsilon^{2}/2 - \epsilon^{3}/3)d}{4}}$$ 

The theorem can be derived using the following lemma:
\begin{lemma}
\label{th:lemma_1}
For any $\epsilon>0$, $\tau < 1/2$ and positive integer $d$, there exists a distribution $\mathcal D$ over $\R^{d\times k}$ for $d=O(\epsilon^{-2}\log{1/\tau})$ such that, for any $\vec{x} \in \R^k$ with $||\vec{x}||_2=1$,
$$P(| \|A\vec{x}\|_2^2 -1  | > \epsilon) < \tau$$
\end{lemma}
by choosing $\tau = 1/m^2$ and by applying the union bound on the vectors $(\vec{u} - \vec{v})/\|\vec{u}-\vec{v}\|_2$ for all vectors $\vec{u}$ and $\vec{v}$ in $V$. It is possible to demonstrate that there is a probability strictly greater than $0$ that a function $f$ exists.

\paragraph{Our Corollary}

Now we can demonstrate that the following lemma holds:
\begin{corollary}
\label{existence_of_f}
For any $0 < \epsilon < 1$ and any integer $m$. Let $d$ be a positive integer such that

$$d \geq 4(\epsilon^{2}/2 - \epsilon^{3}/3)^{-1}\ln{m}$$
Then given the standard basis $E$ of ${\mathbb R}^m$, there is a map $f : {\mathbb R}^m \rightarrow {\mathbb R}^d$ such that for all $\vec{e}_i, \vec{e}_j \in E$, 
\begin{equation}
P(1 - \epsilon<\|f(\vec{e}_i)\|_2^2 <1 + \epsilon) > 1 - \tau = 1 - 1/m^2
\label{eq:p1}
\end{equation}
and
\begin{equation}
P(|f(\vec{e}_i)f(\vec{e}_j)|<2\epsilon)> (1 - \tau)^2 = (1 - 1/m^2)^2
\label{eq:p2}
\end{equation}
\end{corollary}

\begin{proof}
Equation (\ref{eq:p1}) derives from lemma \ref{th:lemma_1} as $\vec{e}_i \in E$ are unitary, that is,  $\|\vec{e}_i\|_2=1$ as $\tau = 1/m^2$.

To prove Equation (\ref{eq:p2}), first, we can observe that $||\vec{e}_i-\vec{e}_j||^2 = ||\vec{e}_i||^2 + ||\vec{e}_j||^2 - 2 \vec{e}_i\vec{e}_j$ = 2 as $\vec{e}_i$ and $\vec{e}_j$ are unitary and orthogonal. Then, we can see that $||f(\vec{e}_i)-f(\vec{e}_j)||^2 = ||f(\vec{e}_i)||^2 + ||f(\vec{e}_j)||^2 - 2 f(\vec{e}_i)f(\vec{e}_j)$.
With Theorem \ref{th:JJL}, the following holds:
$$
2(1-\epsilon)\leq||f(\vec{e}_i)||^2 + ||f(\vec{e}_j)||^2 - 2 f(\vec{e}_i)f(\vec{e}_j)\leq2(1+\epsilon)
$$
Hence:
$$
||f(\vec{e}_i)||^2 + ||f(\vec{e}_j)||^2 -2-2\epsilon\leq 2 f(\vec{e}_i)f(\vec{e}_j)\leq||f(\vec{e}_i)||^2 + ||f(\vec{e}_j)||^2  -2+2\epsilon
$$
Thus, using Equation (\ref{eq:p1}) on the two independent events $f(\vec{e}_i)$ and $f(\vec{e}_j)$:  
$$
P(2 - 2\epsilon -2 -2\epsilon\leq 2 f(\vec{e}_i)f(\vec{e}_j)\leq2+2\epsilon  -2+2\epsilon) =\\
P(|f(\vec{e}_i)f(\vec{e}_j|<2\epsilon) > (1 - \tau)^2
$$

Putting together Equation (\ref{eq:p1}) and Equation (\ref{eq:p2}), it is possible to derive a set $NOV(\varepsilon,\theta)$ of $m$ nearly-orthogonal unit vectors such that for each $\vec{a},\vec{b}\in NOV(\varepsilon,\theta)$:
$$P(\delta(\vec{a},\vec{b}) - \varepsilon \leq \dotprod{\vec{a}}{\vec{b}} \leq \delta(\vec{a},\vec{b}) + \varepsilon) > 1 -\theta$$
by choosing $\varepsilon = 2\epsilon$, a space $\R^d$ with $d=O(\varepsilon^{-2}\log{m})$ and $\theta = 2/m^2 - 1/m^4$.

\end{proof}

\end{document}